\documentclass[twoside,preprint]{article}

\usepackage{aistats2026}

\usepackage[round]{natbib}

\usepackage{cite}
\usepackage{amsmath,amssymb,amsfonts, bbm, amsthm}
\usepackage{tikz}
\usepackage{algorithm}      
\usepackage{algpseudocode}
\usepackage{graphicx}
\usepackage{textcomp}
\usepackage{xcolor}
\usepackage{pgfplots}    
\pgfplotsset{compat=1.18}
\def\BibTeX{{\rm B\kern-.05em{\sc i\kern-.025em b}\kern-.08em
    T\kern-.1667em\lower.7ex\hbox{E}\kern-.125emX}}

\newtheorem{theorem}{Theorem}
\newtheorem{lemma}{Lemma}
\newtheorem{definition}{Definition}
\begin{document}

\twocolumn[

\aistatstitle{Deceptive Exploration in Multi-armed Bandits}

\aistatsauthor{ I. Arda Vurankaya$^{1}$ \And Mustafa O. Karabag$^{1}$ \And  Wesley A. Suttle$^{2}$ }
 \aistatsauthor{Jesse Milzman$^{2}$ \And David Fridovich-Keil$^{1}$ \And Ufuk Topcu$^{1}$}

\aistatsaddress{ $^{1}$The University of Texas at Austin\And $^{2}$ The U.S. Army Research Laboratory } ]

\begin{abstract}

We consider a multi-armed bandit setting in which each arm has a public and a private reward distribution. An observer expects an agent to follow Thompson Sampling according to the public rewards, however, the deceptive agent aims to quickly identify the best private arm without being noticed. The observer can observe the public rewards and the pulled arms, but not the private rewards. The agent, on the other hand, observes both the public and private rewards. We formalize detectability as a stepwise Kullback-Leibler (KL) divergence constraint between the actual pull probabilities used by the agent and the anticipated pull probabilities by the observer. We model successful pulling of public suboptimal arms as a %
Bernoulli process where the success probability decreases with each successful pull, and show these pulls can happen at most at a $\Theta(\sqrt{T}) $ rate under the KL constraint. We then formulate a maximin problem based on public and private means, whose solution characterizes the optimal error exponent for best private arm identification. We finally propose an algorithm inspired by top-two algorithms. This algorithm naturally adapts its exploration according to the hardness of pulling arms based on the public suboptimality gaps. We provide numerical examples illustrating the $\Theta(\sqrt{T}) $ rate and the behavior of the proposed algorithm. 

\end{abstract}

\section{Introduction}

Agents acting in adversarial environments may need to conceal their intentions from other parties using deception. Deceptive strategies have proven useful in various fields such as cybersecurity \citep{cyber} and robotics \citep{robot}, where attackers and defenders may try to hide their true intentions.

We study a scenario where an observer expects an agent to interact with a multi-armed bandit using a reference learning algorithm. Each arm is associated with a public and a private reward distribution. When an arm is pulled, two reward samples, one from the public distribution and one from the private, are generated. The public reward samples are observable to both the agent and the observer, and the private reward samples are observable only to the agent. At each time step, the observer can observe the arm pulled by the agent and the corresponding public reward sample, and can use its observations to decide whether the agent follows the reference algorithm. We consider the case where the reference algorithm is Thompson Sampling \citep{thompson1933likelihood}, meaning that the agent is supposed to learn to optimize public rewards by following Thompson Sampling. Its true intention, however, is to identify the best private arm with high confidence as quickly as possible, which might require deviating from the reference algorithm. We refer to such deviations aiming to identify the best private arm as \textit{deceptive exploration}, that is, performing exploration while looking as if optimizing public rewards to outside parties. 

We consider Thompson Sampling as the reference algorithm because its randomized arm selection allows the agent to manipulate arm pull probabilities to deceptively pursue its own goal. Still, the manipulations should be small so that the behavior of the agent remains plausible to the observer. To formalize plausibility, we enforce a stepwise constraint on the Kullback-Leibler (KL) divergence between the agent's and the reference algorithm's arm selection distributions.

As Thompson Sampling aims to maximize the cumulative reward received over time, its action selection becomes almost purely exploitative while it identifies the best public arm: it explores other arms at an $O(\log T) $ rate, which %
is suboptimal for the agent's exploration goal. Due to posterior concentration, the probability assigned to a suboptimal arm vanishes exponentially fast with each additional pull. Nevertheless, Thompson Sampling cannot assign zero probability to an arm after a finite number of measurements, meaning that the agent can always boost the probability of pulling an arm using its KL budget. By analyzing how the probability assigned to a suboptimal arm changes if the agent repeatedly boosts that probability, we conclude that the agent can keep exploring suboptimal arms at a $\sqrt{T} $ rate as $T \rightarrow \infty. $

Building on this rate result, we explain how to perform deceptive exploration to minimize the time required to identify the best private arm with high confidence, and in particular, how the agent's problem differs from the classical best arm identification problem. The key difference is the asymmetry across arms: some arms are easier to pull than others, which is determined by their distance to the best public arm. We then discuss why, even with this asymmetry, an information balance condition must hold for asymptotic optimality, and how an optimal algorithm would adapt to this asymmetry. We formulate a maximin problem whose solution characterizes the best possible error exponent for deceptive exploration. Finally, we highlight how a simple top-two sampling algorithm with information directed selection  enables quick exploration. 

\textbf{Related Work.} There is an extensive literature on best arm identification \citep{toptwo, track, qin2017improving, shang2020fixed, you2023information}. \citet{toptwo} studies the problem in a Bayesian setting, introduces top-two algorithms, and demonstrates their optimality in terms of the rate of posterior convergence. In a frequentist setting, \citet{track} characterize lower bounds on the sample complexity of the fixed-confidence best arm identification, and introduces the asymptotically optimal Track-and-Stop strategy. \citet{qin2017improving,shang2020fixed} provide optimality guarantees for top-two algorithms in the frequentist fixed-confidence setting. To select between the top-two candidates, \citet{you2023information} introduce information-directed selection, and proves the asymptotic optimality of the resulting algorithm for Gaussian bandits. Similar to these works, we propose a top-two algorithm for deceptive exploration, and in a Bayesian setting, we characterize the best possible error exponent for private best arm identification as a solution to a maximin problem. The key difference between these works and our work is that we constrain the learner's action selection so that its behavior stays plausible to the observer. 

Best arm identification with constraints is studied in different settings, where the learner must satisfy a safety constraint \citep{safe_best}, identify the best arm with limited resources \citep{li2024best}, and identify the best arm covertly \citep{chang2022covert}. The closest to our setting is \citep{chang2022covert}, which formulates the \textit{covert} best arm identification problem, in which an agent wants to hide from an adversary that it is pulling arms while identifying the best arm. Different from our setting, in \citep{chang2022covert}, the adversary only observes generated reward samples, but not the agent's actions. 

We propose a deceptive exploration algorithm by building a connection with best arm identification with costs \citep{kanarioscost, Qin2024OptimizingAE, harding2025balancing}. \citet{kanarioscost} introduce the \textit{Cost Aware Best Arm Identification} problem, in which each arm has an associated stochastic pulling cost, and the goal is to minimize the expected cost required to reach a given confidence level. Similarly, \citet{Qin2024OptimizingAE} study the multi-objective task of simultaneously minimizing regret and identifying the best arm. In our work, successfully pulling an arm takes a random number of time steps, hence each arm can be seen as having a time cost determined by its public suboptimality gap. The difference from the mentioned works is that the time costs in this work are history dependent. 

The considered problem is also related to the problem of privacy in bandits and reinforcement learning \citep{diff_priv,azize2024DifferentialPrivateBAI, deceptive_rl}. \citet{diff_priv} use differential privacy to formalize the notion of privacy in bandits, and design regret minimization algorithms with privacy guarantees. Similarly, \citet{azize2024DifferentialPrivateBAI} study differentially private best arm identification. More related to the scenario of this paper, \citet{deceptive_rl} study the problem of keeping the reward function private from adversaries by being deceptive when learning using a reinforcement learning algorithm. While these works try to keep observed rewards or the reward function private from outside parties, we instead try to deceive outside parties into believing that we are following a reference algorithm.

The use of KL divergence to formalize deceptiveness is common in the literature. \citet{BAI2017251, deception_control, id_con} use KL divergence to formalize deception in control of dynamical systems. 
Different from these works, we study the problem  of being deceptive while learning. \citet{id_con} also study learning a performant policy while being deceptive. In contrast to our setting, where the agent performs active learning, the deceptive agent in \citep{id_con} learns from an offline dataset.

\section{Preliminaries}

\textbf{Notation.} For a positive integer $K $, we denote by $[K] $ the set $\{1,2,...,K\} $. For an event $E $, $\mathbbm{1}_{E} $ denotes its indicator. $\mathcal{S}_K $ denotes the $K $-dimensional probability simplex. We use the standard asymptotic notation $O(\cdot), o(\cdot) $ and $\Theta(\cdot) $ to describe growth rates.

For a set $\mathcal{X}$, $\mathcal{P}(\mathcal{X}) $ denotes the set of all probability measures on $\mathcal{X} $. If $\mathbb{P} $, $\mathbb{Q} \in \mathcal{P}(\mathcal{X}) $, the Kullback-Leibler (KL) divergence between them is defined as $d_{KL}(\mathbb{P},\mathbb{Q}) = \int_{\mathcal{X}} \log(\frac{d\mathbb{P}}{d\mathbb{Q}}) d\mathbb{P}$, where $\frac{d\mathbb{P}}{d\mathbb{Q}} $ is the Radon-Nikodym derivative of $\mathbb{P} $ with respect to $\mathbb{Q} $.
We denote by $Ber(p) $ the Bernoulli distribution with parameter $p $, and by $\mathcal{N}(\mu, \sigma^2) $ the Gaussian distribution with mean $\mu $ and variance $\sigma^2$. The KL divergence between distributions $\mathcal{N}(\mu_1, \sigma^2) $ and $\mathcal{N}(\mu_2, \sigma^2) $ is $\frac{(\mu_1-\mu_2)^2}{2\sigma^2} $.

\subsection{Multiarmed Bandits}

In a multi-armed bandit problem, an agent is presented with $K$ arms, each associated with an \textit{unknown} reward distribution $\{\nu_a\}_{a=1}^K $. At time $t$, the agent chooses an arm $A_t  \in  [K]$, and receives a reward $X_t \sim \nu_{A_t} $. Letting $\mu^*= \max_{a \in [K]}\mathbb{E}_{X \sim \nu_a}[X] $, typically the goal is to minimize the regret $\mathbb{E}[\sum_{t=1}^T (\mu^* - X_t)] $ in $T $ steps.

We restrict attention to Gaussian rewards with a common, known variance $\sigma^2 $. A bandit instance can then be represented as a vector of arms' mean rewards: $ \boldsymbol{\mu} = \{\mu_a\}_{a=1}^K  $,  %
Our results can be extended to other one-dimensional exponential family distributions using corresponding upper and lower bounds on the tail. 

We denote by $\mathcal{H}_t = \{ (A_1,X_1),...,(A_t, X_t) \} $ the history up to time $t$, with $\mathcal{H}_0 = \emptyset $. An (anytime) bandit algorithm $\pi = \{\pi_t\}_{t\geq0} $ is a sequence of decision rules, where each $\pi_t: \mathcal{H}_t \rightarrow \mathcal{S}_K $ maps the history to a distribution over arms. $N_{a,T} $ denotes the (random) number of times the arm $a $ has been pulled before time $T $, that is $N_{a,T} = \sum_{t=1}^{T-1} \mathbbm{1}_{A_{t}=a}  $. We denote by $\hat \mu_{a,T} $ the empirical mean of arm $a $ based on observations before time $T $: $\hat \mu_{a,T} = \frac{1}{N_{a,T}}\sum_{t=1}^{T-1} \mathbbm{1}_{A_{t}=a} X_t $. 

\subsection{Thompson Sampling} \label{subsec: TS}

Thompson Sampling is a Bayesian algorithm for regret minimization.
Let $\mathcal{U} \subset \mathbb{R}^K$ denote the set of possible values of $\boldsymbol{\mu} $.
Thompson Sampling starts with a prior $\Pi_0$ over $\mathcal{U}$ and, after observing $(A_t,X_t)$ at time $t$,
updates its belief via Bayes’ rule:
\begin{equation}
  \rho_t(\boldsymbol{\mu})
  \;=\;
  \frac{\rho_{t-1}(\boldsymbol{\mu})\, p\!\left(X_t\mid \boldsymbol{\mu}, A_t\right)}
       {\int_\mathcal{U} \rho_{t-1}(\boldsymbol{\mu}')\, p\!\left(X_t\mid \boldsymbol{\mu}', A_t\right)\,d\boldsymbol{\mu}'}\,,
\end{equation}
where $\rho_t$ is the density of the posterior $\Pi_t$ and
$p(X_t\mid \boldsymbol{\mu}, A_t)$ is the probability of observing $X_t $ upon pulling arm $A_t $ if the true parameter is $\boldsymbol{\mu} $.

At each time step, Thompson Sampling draws $\hat{\boldsymbol{\mu}_t}  \sim \Pi_{t-1}$ and selects $   A_t \in \arg\max_{a\in [K]} \mathbb{E}[X_a \mid \hat{\boldsymbol{\mu}}]$,
then observes $X_t$ and updates $\Pi_{t-1} $ to $\Pi_t$ as above. Equivalently, Thompson Sampling pulls an arm $a $ with probability that $a $ is optimal under the posterior $\Pi_{t-1} $. 

In this paper, we will be interested in the case where the experiments start with an improper prior over the mean of each arm. At time $t $, the posterior over the mean of arm $a $ will be given as $\mathcal{N}(\hat \mu_{a,t}, \sigma^2/N_{a,t}) $.
Letting $Z_{a,t} $ be a  sample from this posterior, the probability of pulling arm $a$ at time $t $ is $\mathbb{P}(Z_{a,t} = \max_{i}Z_{i,t}) $.

\subsection{Best Arm Identification} \label{subsec: BAI}

Best arm identification is an instance of \textit{pure exploration} problems, and the objective is to identify the arm with the largest expected reward, that is $a^* = \arg \max_{a \in [K]} \mu_a $. This setting is related to, but different from the regret minimization framework, as we do not care about the cumulative reward received over time, but only about the quality of the final decision. 

A key quantity of interest in best arm identification is the posterior probability assigned to the event that an arm $a\neq a^* $ is optimal. Formally, let $\Pi_t $ denote the posterior at time $t $. The probability that arm $a $ is optimal under posterior $\Pi_t $ is $p_{a,t} = \mathbb{P}_{\boldsymbol{\mu}_t \sim \Pi_t}(\mu_{a,t} > \max_{i \neq a} \mu_{i,t}) $. We wish that as the number of observations grows, the quantity $p_{a^*, t} $ approaches 1.

When the rewards are Gaussian, given that each arm is pulled infinitely many times, it can be shown that the error probability (confidence level) is 
\begin{equation*}
    1-p_{a^*, t} = \exp\left(-t \min_{a \neq a^*} \frac{w_{a^*,t}w_{a,t}(\mu_{a^*}-\mu_a)^2}{2\sigma^2(w_{{a^*},t}+w_{a,t})}+o(t)\right),
\end{equation*}
where $w_{a,t} = N_{a,t}/t $ \citep{qin2017improving}. It follows that to get the best possible error exponent, the long term sampling proportions should converge to 
\begin{equation} \label{eq: max_min_gaussian}
    w^* \in \arg\max_{w \in \mathcal{S}_K} \min_{a \neq a^*} \frac{w_1w_{a^*}(\mu_{a^*}-\mu_a)^2}{2\sigma^2(w_{a^*}+w_a)}.
\end{equation}
We remark that \eqref{eq: max_min_gaussian} is a concave maximization problem and has a unique optimal solution $w^* $, and $w^*_a > 0 $ for all $a \in [K] $ \citep{track}.

\section{Problem Statement} \label{sec: PS}

We consider a $K$-armed bandit setting in which an agent is assigned by an observer to interact with the bandit using a reference algorithm $\pi_{\text{ref}} $. In the rest of the paper, $\pi_{\text{ref}} $ denotes Thompson Sampling. 
In the considered setting, pulling an arm $a$ produces two independent reward samples: a \textit{public} reward $X_t^{\text{pub}} \sim \mathcal{N}(\mu_a^{\text{pub}}, \sigma^2) $ observable to both the agent and the observer, and a \textit{private} reward $X_t^{\text{priv}} \sim \mathcal{N}(\mu_a^{\text{priv}}, \sigma^2) $ observable only to the agent. Formally, the complete history $\mathcal{H} _t= \{(A_1, X_1^{\text{pub}},X_1^{\text{priv}}),... ,(A_t,X_t^{\text{pub}},X_t^{\text{priv}})\}$ is only observable to the agent, while the observer can only observe the public portion, i.e.,
$\mathcal{H}_t^{\text{pub}} = \{(A_1,X_1^{\text{pub}}),...,(A_t,X_t^{\text{pub}})\} $, and we have $\pi_{\text{ref}}: \mathcal{H}_t^{\text{pub}} \rightarrow \mathcal{S}_K $. $\Pi^{\text{pub}}_t $ and 
$\Pi^{\text{priv}}_t $ denote the posteriors over the public and private means at time $t $, respectively. 

The observer can monitor the public rewards and the agent’s actions to decide whether the agent follows $\pi_{\text{ref}}$. In contrast, the agent’s goal is to identify the best private arm as quickly as possible with high confidence, which can require deviating from $\pi_{\text{ref}} $. The agent must keep its deviations from $\pi_{\text{ref}} $ small to ensure that its behavior remains plausible to the observer. We formalize plausibility by enforcing a stepwise constraint on the KL divergence between the agent's and $\pi_{\text{ref}} $'s arm selection distributions. The use of KL divergence to quantify deceptiveness is common in the literature \citep{BAI2017251}.

Let $a^* $ denote the (unique) best private arm, and let $p^*_{t} $ denote the probability that $a^* $ is optimal under $\Pi_t^{\text{priv}} $.  
We formally state the problem as follows: 

\textbf{Problem}. Consider a $K$-armed Gaussian bandit with public means $\{\mu_a^{\text{pub}}\}_{a=1}^K $ and private means $\{\mu_a^{\text{priv}}\}_{a=1}^K $, and known variance $\sigma^2 $. Design an exploration algorithm $\pi:\mathcal{H}_t \rightarrow\mathcal{S}_K$ that maximizes the asymptotic decay rate of $1- p^*_{t} $ subject to a per-step KL constraint: for all $t> 0 $ and any history $\mathcal{H}_t $, $d_{KL}(\pi (\mathcal{H}_t), \pi_{\text{ref}}(\mathcal{H}_t^{pub})) \leq \epsilon$.

As Thompson Sampling needs to explore all arms to minimize the regret, 
it is natural to ask whether the exploration performed by $\pi_{\text{ref}} $ suffices for best private arm identification. The answer is no for two reasons. First, it is well known in the bandit literature \citep{pure_explore} that regret minimization algorithms are suboptimal for pure exploration, especially for reaching high confidence levels. The reason is that regret minimization algorithms try to optimize cumulative reward, and hence they switch to an exploitation phase in which they pull the empirically best arm almost exclusively. Concretely, Thompson Sampling would explore suboptimal arms at an $O(\log T) $ rate, %
whereas classical asymptotically optimal best arm identification would require pulling an arm $a $ roughly $w_a^*T $ times in $T$ steps, where $w^* = \{w_a\}_{a=1}^K $ is an instance-dependent allocation vector (see Section \ref{subsec: BAI}).
Second, in our setting, this issue is amplified by the potential misalignment between public and private rewards, so $\pi_{\text{ref}} $'s exploration about public rewards may provide little useful information about the best private arm. In that case, it may not be possible to reach even moderate confidence levels by just following $\pi_{\text{ref}}$. For example, consider a case where the first and the second best private arms are very suboptimal for public rewards. In that case, $\pi_{\text{ref}}$ could establish their suboptimality with very few samples, and would pull them very infrequently after that point (at the $O(\log T) $ rate), making it very hard for the agent to decide between these two arms even with a moderate confidence level.

We now explain why these necessary deviations from $\pi_{\text{ref}} $ are also feasible. %
$\pi_{ref} $ uses a randomized arm selection rule, where the randomization comes from the posterior sampling step (see Section \ref{subsec: TS}), so $\pi_{\text{ref}} $ assigns strictly positive probability to every arm whenever the public posterior retains uncertainty. In other words, no suboptimal arm can be completely ruled out with a finite number of pulls---its selection probability can decay but remains nonzero. By choosing at each time $t$ an action distribution within a KL ball of $\pi_{\mathrm{ref}}(\mathcal{H}_t^{\text{pub}})$, the agent can increase these probabilities to obtain additional exploratory samples while keeping its behavior plausible to the observer.

\textbf{Remark:} \textit{We take Thompson Sampling as the reference algorithm due to its randomized nature, which the agent can exploit to perform deceptive exploration. Suppose that $\pi_{\text{ref}} $ is a deterministic algorithm such as UCB \citep{auer2002finite}, so it is not possible to deviate from $\pi_{\text{ref}} $. Still, we can introduce the following model for public reward observations: the agent observes $X_{1,t}^{\text{pub}} \sim \nu_{A_t}^{\text{pub}} $ and the observer observes $X_{2,t}^{\text{pub}} \sim \nu_{A_t}^{\text{pub}} $. Since the observer does not know the agent's observations with certainty, it cannot expect the agent to take deterministic actions, but instead may maintain a belief over the agent's empirical mean observations to predict the agent's next action. From the observer's perspective, $\pi_{\text{ref}} $ would be a randomized algorithm that closely resembles Thompson Sampling.}

\section{Characterizing the Achievable Rate of Pulls Under a KL Constraint} \label{sec: Rate}

As mentioned in Section \ref{sec: PS}, $\pi_{\text{ref}} $ pulls public suboptimal arms at an $O(\log T) $ rate, which is far from optimal for identifying the best private arm. In this section, we answer the following: \textit{ If the agent is allowed to deviate from $\pi_{\text{ref}} $ with a KL budget of $\epsilon $ at each time step, at what rate can it continue to pull public suboptimal arms as $T \rightarrow \infty $? }

In this section, without loss of generality, we assume that arm $1 $ is the best public arm. Also for notational brevity, we denote $\mu_a^{\text{pub}} $ by $\mu_a $. Proofs of all technical results are deferred to the appendix. 

We analyze the behavior of $\pi_{\text{ref}} $ as $T \rightarrow \infty $ on the event that each arm is pulled infinitely often (which is a necessity for asymptotically optimal best arm identification).
At time $t $, the posterior over arm $a $ has the distribution $\mathcal{N}(\hat \mu_{a,t}, \sigma^2/N_{a,t}) $. Let $Z_{a,t} $ denote the sample from this posterior distribution at time $t $: $Z_{a,t} \sim \mathcal{N}(\hat \mu_{a,t}, \sigma^2/N_{a,t}) $. Then $\pi_{\text{ref}} $ assigns probability $\pi_{a,t} =  \mathbb{P}(Z_{a,t} > Z_{i,t} \text{ for all } i \neq a) $ to arm $a $ at time $t $. Since posteriors are well concentrated, it can be shown that $c_t \mathbb{P}(Z_{a,t} > Z_{1,t}) \leq  \pi_{a,t} \leq \mathbb{P}(Z_{a,t} > Z_{1,t}) $ where $c_t \rightarrow 1$. %
Noting that even if the agent deviates, the best public arm (arm 1) is pulled almost always, that is the ratio $N_{1,t} / N_{a,t} $ diverges as $t \rightarrow \infty $, and that $\hat \mu_{i,t} \rightarrow \mu_i$ for all $ i$, we have $\pi_{a,t} = \exp(-\frac{N_{a,t}(\mu_1-\mu_a)^2}{2\sigma^2} +  o(N_{a,t})) $. %

At each time step, the agent can use its KL budget to increase the probability of pulling a suboptimal arm, but there is a caveat: $\pi_{a,t} $ shrinks at an exponential rate in $N_{a,t}$.
The following lemma is central to our analysis, as it shows that a small probability $p $ can be exponentially increased within a KL ball of radius $\epsilon $. 

\begin{lemma}\label{lem:kl_boost}
    Let $\mathbb{P} $ be a probability measure on a set $\mathcal{X} $, $E $ be an event, $p = \mathbb{P}(E) >0$, and $\epsilon > 0 $. Consider
\begin{equation}
    \begin{aligned}
        q^*(p) = \max_{\mathbb{Q} \in \mathcal{P}(\mathcal{X})} \quad &\mathbb{Q}(E) \\
        \text{s.t.} \quad  &d_{KL}(\mathbb{Q},\mathbb{P}) \leq \epsilon.
    \end{aligned} \end{equation} 
    Then,
    \begin{equation}
         \lim_{p \rightarrow 0} \frac{\log(\epsilon/p)}{\epsilon / q^*(p)}  = 1.   
    \end{equation}

\end{lemma}

\begin{definition}
    We say that the agent \textit{boosts} arm $a $ at time $t $ if it determines the action distribution by solving
\begin{equation} \label{eq: boost}
    \begin{aligned}
        \max_{\pi} \quad &\pi_a \\
        \text{s.t. } &\pi \in \mathcal{S}_K, d_{KL}(\pi, \pi_{\text{ref}}(\mathcal{H}_t^{pub})) \leq \epsilon, 
    \end{aligned}
\end{equation}
that is, the agent specifically maximizes the probability of pulling arm $a $ using its KL budget. 
\end{definition}

Using Lemma \ref{lem:kl_boost}, we can conclude that if the agent keeps boosting a suboptimal arm $a $, the probability of pulling arm $a $ instead decreases as $\pi_{a,t} \sim \frac{\epsilon}{N_{a,t}(\mu_1-\mu_a)^2/2\sigma^2 + o(N_{a,t})} $. Although the probability of pulling arm $a $ still decreases with each pull, the decrease is harmonic instead of exponential, given that the agent can keep boosting this probability. 

Based on the above discussion, the situation as $T \rightarrow \infty  $ can be modeled as a sequence of Bernoulli trials, where the success probability decreases as a function of the number of prior successes, and remains constant between successes. Here, a success corresponds to the pull of a public suboptimal arm $a $. The following lemma characterizes the growth rate of the number of successes in such a process:

\begin{lemma} \label{lem: bernoulli}
    Consider a sequence of Bernoulli trials in which the probability of success at time $t $ is given by $\frac{c}{M_t + M_0} $, where $M_t $ is the number of successes achieved before time $t $, and $0 < c\leq M_0$ . Then,
    \begin{equation}
        \lim_{T \rightarrow \infty } \frac{M_T}{\sqrt{2cT}} = 1 \quad \text{a.s.}.
    \end{equation}
\end{lemma} 
Combining Lemma \ref{lem: bernoulli} with the discussion of how the probability of pulling a suboptimal arm decreases with each pull, we arrive at the following theorem. %

\begin{theorem} \label{thm: pull rate}
    Let $a $ be suboptimal according to public rewards, and suppose that the agent repeatedly boosts arm $a $ with its KL budget of $\epsilon $, i.e. it solves \eqref{eq: boost} at each time step $t $. If each arm is pulled infinitely often, then
    $$
    \lim_{T \rightarrow \infty} \frac{N_{a,T}}{\sqrt{T}} = \sqrt{\frac{4\epsilon \sigma^2}{(\mu_1 -\mu_a)^2}} \quad \text{a.s.}.
    $$
\end{theorem}

Theorem \ref{thm: pull rate} shows that the agent can keep pulling a public suboptimal arm at a $\sqrt{T} $ rate by deviating from $\pi_{\text{ref}} $, a major improvement over the usual $\log T $ rate of Thompson Sampling. Moreover, we see that the rate scales linearly with $\sqrt{\epsilon} $, and inversely with the public suboptimality gap of arm $a$. This result aligns with the intuition that deceptive exploration becomes easier with increasing KL budget, and arms with larger public suboptimality gaps are harder to explore.

\section{How to Deceptively Explore}

In this section, we discuss how to perform deceptive exploration for quick best private arm identification. We first discuss the general characteristics that an asymptotically optimal exploration algorithm should have. Then we formulate a maximin problem whose solution characterizes how much boosting effort the agent needs to allocate to each arm to get the best possible error exponent. Finally, we discuss how a simple top-two sampling procedure results in efficient exploration.   

Let arm 1 be the best private arm, and $i^* $ be the best public arm. With a change of notation, we denote by $\Delta_a^{\text{priv}} $ the suboptimality gap of arm $a $ in terms of private rewards: $\Delta_a^{\text{priv}} = \mu_1^{\text{priv}} -\mu_a^{\text{priv}} $. Similarly, we have $\Delta^{\text{pub}}_a  = \mu_{i^*}^{\text{pub}} - \mu_a^{\text{pub}} $, the public suboptimality gap of arm $a$. Finally, for notational brevity, we assume the reward variance $\sigma^2 = 1 $.

As explained in Section \ref{sec: PS}, the agent's goal is to identify the best private arm as quickly as possible while keeping its behavior plausible to the observer. Our setting differs from the classical best arm identification setting in two key ways. First, in our case, due to the KL constraint, pulling a desired arm has an associated success probability, which is determined by $\pi_{\text{ref}} $ and the KL budget $\epsilon $. Consequently, one can only succeed in pulling an arm after a random number of time steps. Second, there is asymmetry across arms: those with smaller public suboptimality gaps are easier to pull, meaning the expected time to obtain an additional pull is shorter. In classical best arm identification, the goal is to minimize the sample complexity to reach a certain confidence level, which is equivalent to minimizing the time. In our case, this equivalence does not necessarily hold, as pulling different arms requires different amounts of time steps on average. 

\textbf{Characteristics of Optimal Deceptive Exploration}. As mentioned in Section \ref{subsec: BAI}, for Gaussian rewards, the leading term in the error exponent is
\begin{equation} 
\label{eq: Exponent}
    {-\min_{a\neq1}\frac{N_{1,t}N_{a,t} (\Delta_a^{\text{priv}})^2}{2(N_{1,t} + N_{a,t})}}.
\end{equation}
The structure of the exponent is crucial to understanding how to perform exploration. The quantity $\frac{N_{1,t}N_{a,t}(\Delta_a^{\text{priv}})^2}{N_{1,t} + N_{a,t}} $ is the \textit{evidence} gathered for arm $a $, that is, a measure of information collected to declare that arm $ 1$ is better than arm $ a$. Since the error exponent is governed by the arm for which we gathered minimum evidence, it follows that one needs to gather an equal amount of evidence for each arm. In other words, if our main source of uncertainty is whether arm 1 is better than arm $a $, we cannot resolve this uncertainty by generating data from some other arm $b \in [K] \setminus\{1,a\} $. Our exploration algorithm should aim to equalize the evidence gathered for each arm. We refer to this property as \textit{information balance}.

Although an exploration algorithm should equalize evidence across arms, note that one can reduce the uncertainty about whether arm 1 is better than arm $a $ either by sampling from arm $1 $ or sampling from arm $a $. We expect that an optimal algorithm should balance the \textit{information gains} of arm 1 and arm $a $ with their hardness before deciding which one to pull. It might be possible to reach higher confidence levels more quickly by partially compensating for the harder arm by generating more samples from the easier arm.

To gain a quantitative understanding of the problem, first note that as $T \rightarrow \infty $, the public optimal arm $i^* $ is pulled almost exclusively, and thus asymptotically, $i^* $ can never become a bottleneck for the agent's private exploration. Consequently, the agent would not need to boost $i^* $ in the long run, and instead would allocate its boosting effort among the public suboptimal arms. Note that this does not mean the agent never boosts $i^* $: initially the identity of the public optimal arm is unknown to the agent, hence the agent might boost it if suggested by the exploration algorithm. 
Let $w_{a,T} $ denote the fraction of times that the agent boosts arm $a $ in $T $ steps as $T \rightarrow \infty $, and let $w_T = (w_{a,T})_{a\neq i^*} \in \mathcal{S}_{K-1} $. Then the effective time that passes for arm $a $ in $T $ steps is $Tw_{a,T} $. Using Theorem \ref{thm: pull rate}, we have $N_{a,T} \approx \sqrt{\frac{4\epsilon Tw_{a,T}}{(\Delta_a^{\text{pub}})^2}} $ when $T $ is large. To retain the characteristics of the best arm identification problem, we assume that the best public and private arms do not coincide, that is $ i^* \neq 1$. The exponent in \eqref{eq: Exponent} can be written as $ -\sqrt{4 \epsilon T} \text{ } \Gamma(w_T) $, where $\Gamma(w)$ is equal to
\begin{equation} \label{eq: Gamma}
     \min\left\{ \min_{a \not\in \{1, i^*\}}   \frac{\sqrt{w_1w_a}(\Delta_a^{\text{priv}})^2}{\sqrt{w_1}\Delta_a^{\text{pub}} + \sqrt{w_a}\Delta_1^{\text{pub}}}, \frac{\sqrt{w_1}(\Delta_{i^*}^{\text{priv}})^2}{\Delta_1^{\text{pub}}}\right\} .
\end{equation}

It follows that to get the best possible error exponent, the long-term boosting proportions should converge to
\begin{equation} \label{eq: max_min_all}
    w^* \in \arg \max_{w \in \mathcal{S}_{K-1}} \Gamma(w).
\end{equation}
We make the following observations about \eqref{eq: max_min_all}. It is a concave maximization problem, and has a unique optimal solution, which we denote by $w^* $. Furthermore, the optimal solution satisfies,  for all $a, a' \in [K] \setminus\{i^*\} $,
\begin{equation*} 
    \frac{\sqrt{w_1^*w_a^*}(\Delta_a^{\text{priv}})^2}{\sqrt{w_1^*}\Delta_a^{\text{pub}} + \sqrt{w_a^*}\Delta_1^{\text{pub}}} = \frac{\sqrt{w_1^*w_{a'}^*}(\Delta_{a'}^{\text{priv}})^2}{\sqrt{w_1^*}\Delta_{a'}^{\text{pub}}+\sqrt{w_{a}^*}\Delta_1^{\text{pub}}}  .
\end{equation*}

This equality formalizes the information balance condition discussed above.
Note that we have $\frac{\sqrt{w_1^*}(\Delta_{i^*}^{\text{priv}})^2}{\Delta_1^{\text{pub}}}\geq\tfrac{\sqrt{w_1^*w_a^*}(\Delta_a^{\text{priv}})^2}{\sqrt{w_1^*}\Delta_a^{\text{pub}} + \sqrt{w_a^*}\Delta_1^{\text{pub}}} $ at the optimum. Please refer to the appendix for analysis of \eqref{eq: max_min_all}. 

\textbf{Deceptive Exploration with Top-two Sampling.} Since the long-term boosting proportions should converge to $w^* $, one can use a tracking-based approach \citep{track} to design an exploration algorithm. In more detail, at each time step, we can form an estimate $\hat{w} $ of $w^* $ by solving \eqref{eq: max_min_all} using the empirical estimates of the quantities $\Delta_a^{\text{priv}} $ and $\Delta_a^{\text{pub}} $, and then decide which arm to boost to make our current boosting proportions closer to $\hat{w} $. For computational efficiency, we instead use the much simpler but still asymptotically optimal top-two approach.

\begin{algorithm}[t]
\caption{Deceptive Exploration with Top-Two Sampling}
\label{alg:ttts-bai}
\small
\begin{algorithmic}[1]
\Require Number of arms $K$, public means $\boldsymbol{\mu}^{\text{pub}} $ and private means $\boldsymbol{\mu}^{\text{priv}} $, KL budget $\epsilon $,  confidence $\delta \in (0,1)$
\Ensure Recommendation $\hat{a} \in \{1,\dots,K\}$
\State Initialize public and private priors, $\Pi_0^{\text{pub}} $, and $\Pi_0^{\text{priv}} $  for each arm $k$
\For{$t = 0,1,\dots$}
  \State \textbf{(Leader Draw)} Sample $\theta \sim \Pi_t^{\text{priv}} $ 
  \State $l \gets \arg\max_{k} \theta_k^{}$ \Comment{Leader}
  \State \textbf{(Challenger Draw)} Sample $c \neq l $ with probability proportional to $c $ is optimal under $\Pi_t^{\text{priv}}$\Comment{Challenger}
  \State \textbf{(Top-two choice)} With probability $\frac{N_{c,t}}{N_{c,t}+N_{l,t}} $, set $\hat i_t \gets l$; otherwise set $\hat{i}_t \gets c$ \Comment{Arm to boost}
  \State Compute $\pi_{\text{ref}, t} $ from the posterior $\Pi_t^{\text{pub}} $
  \State Boost arm $\hat{i}_t $ with KL budget $\epsilon $ to get $\pi_{\text{ref},t}' $
  \State Sample $A_t \sim  \pi_{\text{ref},t}' $ and pull arm $A_t $
  \State Observe $(X_t^{\text{pub}}, X_t^{\text{priv}}) $, update $\Pi_t^{\text{pub}}, \Pi_t^{\text{priv}} $

  \If{$\max_a \mathbb{P}_{\theta \sim \Pi_t^{\text{priv}}}(\theta_a =\max_i \theta_i) \geq 1 - \delta $}
     \State \Return $\hat{a} \gets \hat{a}_t$ \Comment{Stop when posterior best-arm prob.\ is high enough}
  \EndIf
\EndFor
\end{algorithmic}
\end{algorithm}

Our proposed exploration algorithm is given in Algorithm \ref{alg:ttts-bai}. At each time step, we first choose two arms: a \textit{leader} and a \textit{challenger}. The leader, denoted by $l $, is determined by first sampling $\theta \sim  \Pi_t^\mathrm{priv} $, and letting $l = \arg \max_k \theta_k $. Then a challenger $c \neq l $ is sampled with probability proportional to $c $  being optimal under $\Pi_t^{\text{priv}} $. Given sufficient initial exploration, with high probability, the leader will coincide with the true best private arm. The probability that the arm $a $ is the challenger will be approximately $\exp\left(-\frac{N_{a,t}N_{l,t}(\mu_l-\mu_i)^2}{N_{a,t}+N_{l,t}}\right) $, meaning that the arms for which we gathered less evidence will have a higher probability of being the challenger. This selection rule naturally seeks information balance. Note that if the reference probabilities are given, Algorithm \ref{alg:ttts-bai} has the same time and space complexity as the standard top-two algorithms, with the addition of a single-parameter KL-constrained maximization problem, which can be efficiently solved with bisection. The probabilities $\pi_{\text{ref}, t} $ can be estimated by repeated sampling or numerical integration. 

Having identified the leader and the challenger, we then need a rule to select between them. To see how, we form a link between our problem and the ``Best Arm Identification with Costs'' problem studied in \citep{Qin2024OptimizingAE}. We take the following alternative view of the exploration process: the agent decides to pull an arm, and then pays a time cost, that is, the (random) time it needs until successfully pulling that arm. Noting that if the probability of successfully pulling an arm $a$ is $p_a $ upon boosting, the expected time cost is roughly $1/p_a $, we arrive at the following selection rule: with probability $\frac{N_{c,t}p_l}{N_{c,t}p_l+N_{l,t}p_c} $ pull the leader, otherwise pull the challenger. This rule balances information gains of the leader and the challenger with their respective time costs. Since in reality, the agent does not pull an arm but rather boosts it, to match these probabilities conditional on successfully gathering an exploratory sample, we arrive at the following boosting rule: with probability $\frac{N_{c,t}}{N_{c,t}+N_{l,t}} $ boost the leader, otherwise boost the challenger. 

Note that the selection rule we propose to decide between the leader and the challenger is the same as the information-directed selection rule proposed in \citep{you2023information}. A natural question is: in which step does Algorithm \ref{alg:ttts-bai} adapt its exploration to the hardness of pulling arms? The answer is the randomization between the leader and the challenger. Boosting probabilities are determined by the respective information gains, and actual pull probabilities assigned to the leader and the challenger are reweighted versions of the information gains by probabilities allowed by $\pi_{ref} $. While the algorithm seeks to gather equal evidence from each arm, the randomization between the leader and the challenger naturally favors easier arms.

\section{Numerical Experiments}
In this section, we present numerical examples demonstrating both the achievable rate of pulls given by Theorem \ref{thm: pull rate} and the behavior of Algorithm \ref{alg:ttts-bai}. All plots are presented with $95 \% $ confidence intervals.

\textbf{Demonstrating the Rate from Theorem \ref{thm: pull rate}.} We consider a setting with $\epsilon = 0.1 $ and public means $\boldsymbol{\mu^{\text{pub}}} = (0.6, 0.3, 0.0, 0.2) $. The agent boosts each suboptimal arm equally many times, that is $w_{a,t} = 1/3 $ for $a = 2,3, 4 $. Theorem \ref{thm: pull rate} predicts that the number of pulls for arm $a$ at time $t$ grows as $\phi_{a,t}:= \sqrt{\frac{4\epsilon w_{a,t}t}{(\mu_1-\mu_a)^2}}$. We compare $\phi_{a,t}$ with the empirical number of pulls from a public suboptimal arm.
\begin{figure}[h]
    \centering    \includegraphics{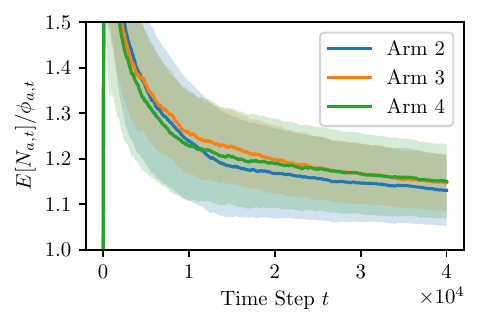}
    \caption{Comparison of the actual rate with the  rate in Theorem \ref{thm: pull rate}. Results are averaged using 50 random seeds. Early values beyond the plotting range are not shown.}
    \label{fig:rate}
\end{figure}

We see that the $\phi_{a,t}$ rate underestimates the actual expected value, but they get closer as time increases. This is expected, since $\phi_{a,t} $ is an asymptotic rate that becomes relevant after posterior concentration happens. During the initial exploration phase, suboptimal arms can be pulled more frequently than this rate.

\textbf{Impact of KL Budget on Confidence.} We investigate how quickly the agent can identify the best private arm by following Algorithm \ref{alg:ttts-bai}, as the KL budget $\epsilon  $ takes the values $\{0,10^{-3}, 10^{-2}, 10^{-1}, 1, \infty \} $. In more detail, in Figure \ref{fig:eps_comparison}, we plot the probability that an arm other than $i^*_{\text{priv}} $ is optimal under the posterior $\Pi_t^{\text{priv}} $. We denote this error probability by $1-p_t^* $. The public and private means are given by $\boldsymbol{\mu^{\text{pub}}} = (0.6, 0.3, 0.0, 0.2)$ and $\boldsymbol{\mu^{\text{priv}}} = (0.2, 0.5, 0.1, 0.0)$, respectively. 

\begin{figure}[h]
\begin{minipage}{\columnwidth}
\includegraphics{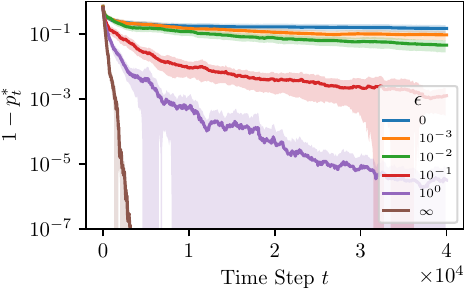}
    \caption{Error probabilities for best arm identification with changing KL budget $\epsilon $. Results are averaged over 100 random seeds.}
    \label{fig:eps_comparison}
\end{minipage}
\end{figure}

We observe that, the agent can reach higher confidence levels faster as the KL budget $\epsilon $ increases. This is expected, as the rate with which the agent can explore public suboptimal arms increases with $\epsilon $. Note that $\epsilon=0 $ corresponds to following $\pi_{\text{ref}} $ without any deviation. For $\epsilon = \infty $, there is no constraint on the agent's behavior, meaning that in this case, we recover the original behavior of the top-two algorithm, and the error probability decreases exponentially fast with time, while the rate is subexponential for the other values.

\textbf{Impact of Asymmetry.} We investigate how the difference in public suboptimality gaps among the arms affects the behavior of Algorithm \ref{alg:ttts-bai}. In more detail, we fix the private means, and consider two cases. In the first case, all public suboptimal arms have the same suboptimality gap. 
In the second case, we change the public means so that each arm has a different public suboptimality gap and arms 3 and 4 become harder to pull than arm 2. In Figure \ref{fig:asym}, we ran this experiment for private means given by $\boldsymbol{\mu}^{\text{priv}}=(0.2, 0.5, 0.0, 0.0)$. The public means in the two cases are given by $\boldsymbol{\mu}_1^{\text{pub}}=(0.6, 0.1, 0.1, 0.1)$ and $\boldsymbol{\mu_2^{\text{pub}}}=(0.6, 0.5, 0.0, 0.3) $, and we have $\epsilon=0.1 $. 

\begin{figure}[h] 
    \centering  \includegraphics[width=0.8\columnwidth]{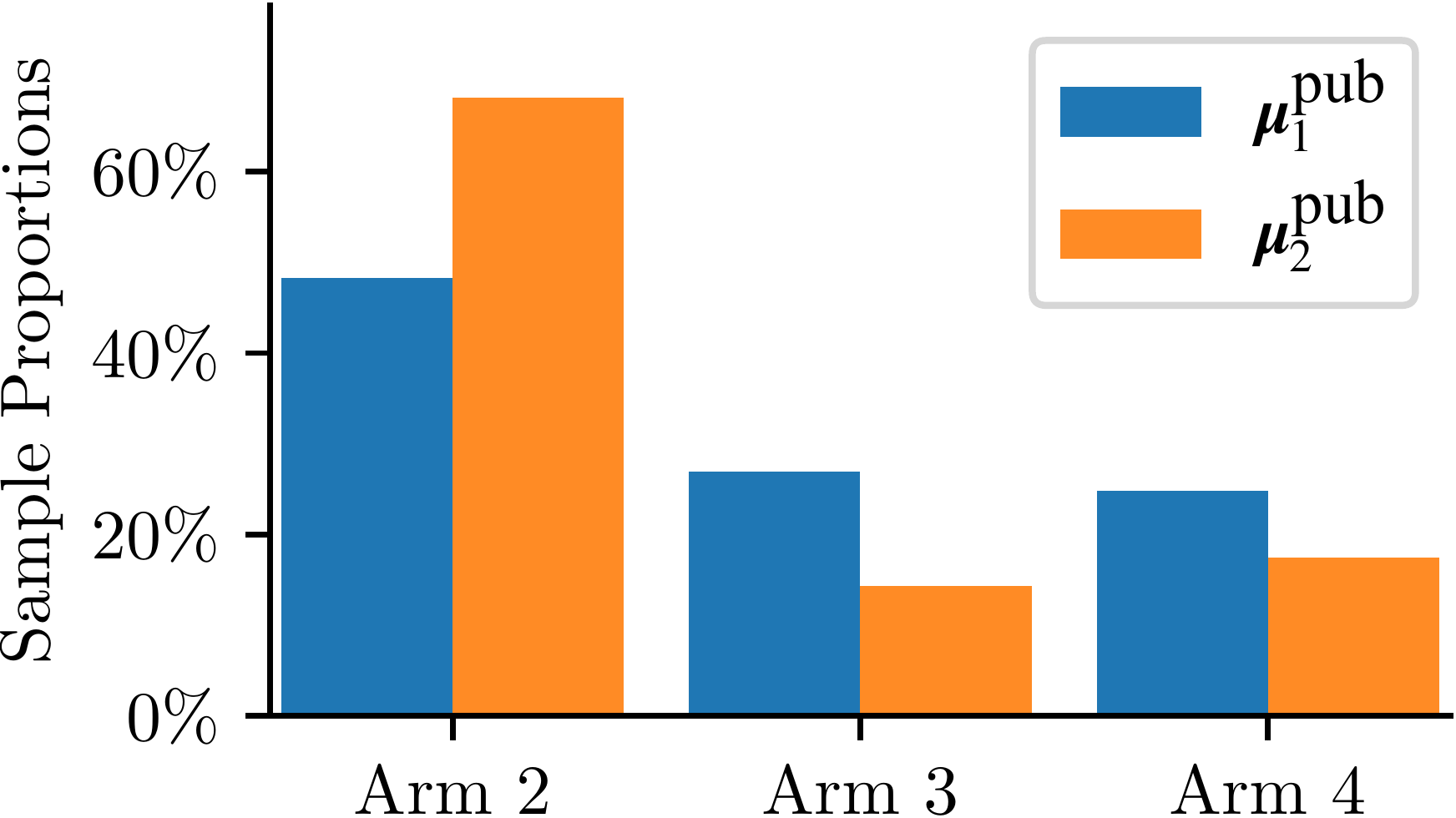}
    \caption{Impact of public suboptimality gaps on the proportion of the total number of samples generated from each arm for the same private means $\boldsymbol{\mu^{\text{priv}}}$. The proportions are averaged over 100 random seeds.}
    \label{fig:asym}
\end{figure}

We see that when pulling arms 3 and 4 becomes harder compared to arm 2, Algorithm \ref{alg:ttts-bai} generates more samples from arm 2, partially compensating for the harder arms. Although arms 3 and 4 share the same private means, arm 3 is pulled less frequently compared to arm 4, as it is harder to pull.

\textbf{Convergence to $\Gamma^* $ under Algorithm \ref{alg:ttts-bai}.} We investigate the convergence of the quantity $\Gamma(w_t) $ defined in Equation \eqref{eq: Gamma} to $\Gamma^* = \max_{w \in \mathcal{S}_{K-1}} \Gamma(w) $ on two bandit instances. Public means are given by $\boldsymbol{\mu}^{\text{pub}}_1 = (0.6, 0.3, 0.1) $ and $\boldsymbol{\mu}^{\text{pub}}_2 = (0.8, 0.3, 0.6) $, while private means are given by $\boldsymbol{\mu}_1^{\text{priv}} = (0.2, 0.5, 0.3) $ and $\boldsymbol{\mu}_2^{\text{priv}} = (0.1, 0.35, 0.2)  $. Figure \ref{fig: gamma} shows the value of $\Gamma^* -\Gamma(w_t) $ for $\epsilon = 0.1 $, which approaches $0$ for both instances.

\begin{figure}[tp]
    \centering    \includegraphics[width=0.9\columnwidth]{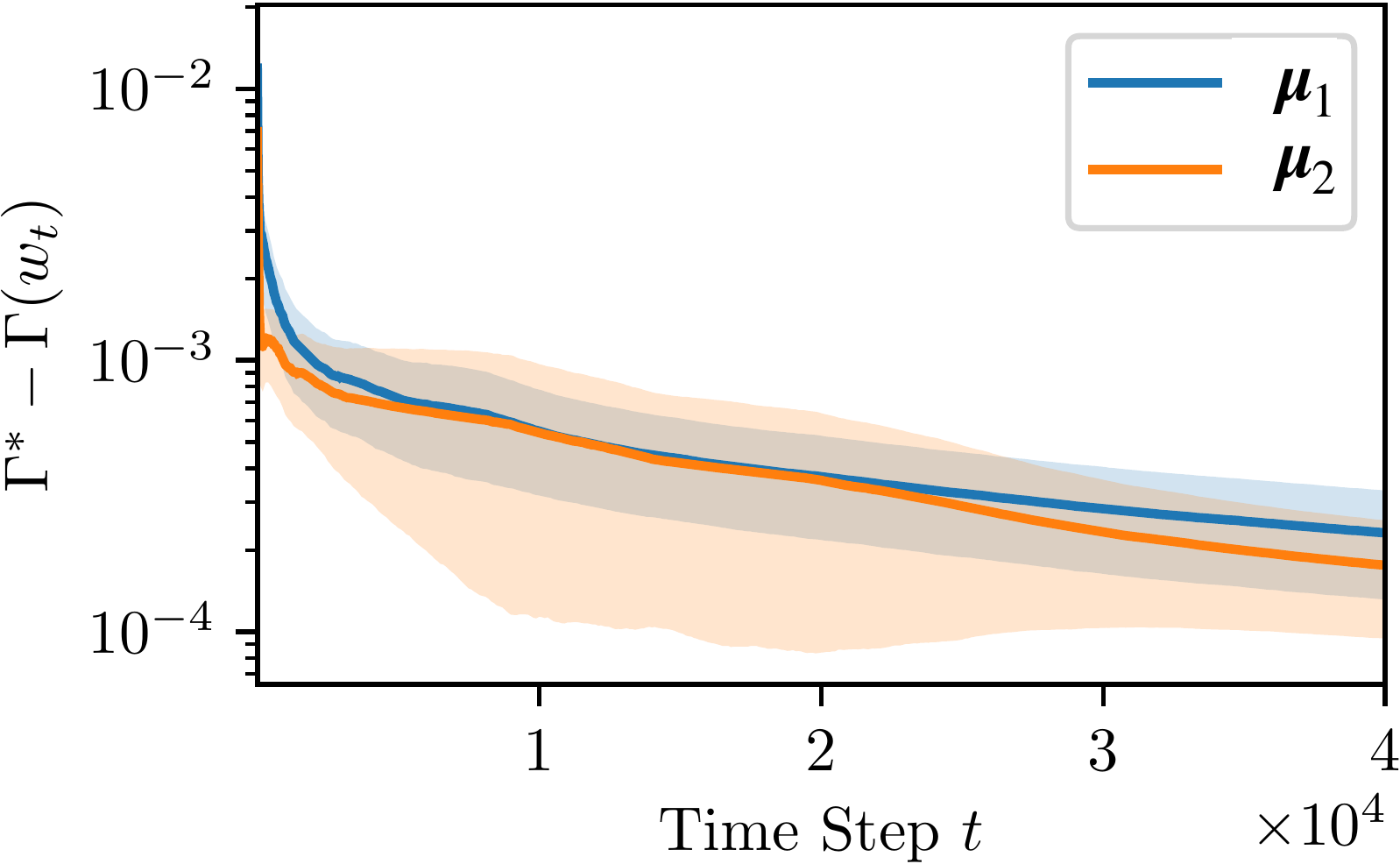}
    \caption{Convergence of $\Gamma(w_t) $ to the optimal exponent $\Gamma^* $ for two different bandit instances. Results are averaged using 50 seeds.}
    \label{fig: gamma}
\end{figure}

\section{Conclusion}

We studied the problem of deceptive exploration in multi-armed bandits. We showed that the agent can keep exploring public suboptimal arms at most at a $\Theta(\sqrt{T}) $ rate when there is a per-step KL constraint on its action selection. We then characterized the best possible error exponent for best private arm identification as a solution to a maximin problem, and proposed an exploration algorithm based on top-two sampling.

Future work involves extending our results to regret minimization or other pure exploration problems. Another possible direction is studying deceptive exploration in scenarios with more complex information structures, for instance, a case in which the agent does not directly observe private rewards, but public and private rewards are correlated. %

\bibliography{ref}

\clearpage
\appendix
\thispagestyle{empty}

\onecolumn
\aistatstitle{Deceptive Exploration in Multi-armed Bandits: \\
Supplementary Materials}

\section{Proofs of Technical Results}

\subsection{Proof of Lemma \ref{lem:kl_boost}}

For notational clarity, we denote the KL divergence between two Bernoulli distributions $d_{KL}(Ber(q), Ber(p)) $ as $d_{KL}(q, p) $:
\begin{equation}
    d_{KL}(q, p) = q\log \frac{q}{p} + (1-q) \log \frac{1-q}{1-p}
\end{equation}

We will show that the maximization problem in Lemma \ref{lem:kl_boost} can be simplified into a single parameter KL-constrained maximization problem.

Let $q = \mathbb{Q}(E) $. Note that we have
\begin{equation}
    d_{KL}(\mathbb{Q, P}) = d_{KL}(\mathbb{Q}(\cdot |E ), \mathbb{P}(\cdot |E ))\mathbb{Q}(E) + d_{KL}(\mathbb{Q}(\cdot |E^c ), \mathbb{P}(\cdot |E^c ))\mathbb{Q}(E^c) + d_{KL}(q, p).
\end{equation}
It follows that the optimal $\mathbb{Q}^* $ must satisfy $\mathbb{Q}^*(\cdot | E) = \mathbb{P}(\cdot |E) $ and  $\mathbb{Q}^*(\cdot | E^c) = \mathbb{P}(\cdot |E^c) $. Thus the problem in Lemma $\ref{lem:kl_boost} $ reduces to 
\begin{align}
    q^*(p) = \max \quad  &q \\
    s.t.  \quad &d_{KL}(Ber(q), Ber(p)) \leq \epsilon.
\end{align}

First, we will show that for sufficiently small $p $, $q^*(p) < 1/2 $. For any $q > 0 $, we have 
\begin{align}
    d_{KL}(q,p) &= q \log \frac{q}{p} + (1-q)\log (1-q) - (1-q) \log (1-p) \\
    &\geq q\log  \frac{q}{p} + (1-q) \log(1-q) \rightarrow \infty \text{ as } p \rightarrow 0
\end{align}
Hence for sufficiently small $p $, we have $q^*(p) < 1/2. $

Note the following inequalities for $x,p \in [0, 1/2]$:
\begin{align}
    &-x -x^2 \leq \log(1-x) \leq -x, \\
    & p \leq -\log(1-p) \leq p/(1-p) \leq 2p
\end{align}

Then for $ p \leq 1/2$ and $q \leq 1/2 $, we have
\begin{equation}
    q\log(q/p) - q-q^2 \leq d_{KL}(q,p) \leq q\log(q/p)-q+q^2+2p
\end{equation}

Define $L = L(p) = \log(\epsilon /p) $. Note that $L \rightarrow \infty  $ as $p \rightarrow 0 $.

The proof will proceed in two steps. First we will show that for sufficiently small $p $, we have $q^*(p) \geq \epsilon/L $. Then, we will show that for sufficiently small $p $, $q^* <  \epsilon/(L-a\log L)$ for $a > 1 $. The result will then follow from these upper and lower bounds. 

We will first show that the point $q_0 = \epsilon/L $ is feasible for sufficiently small $p $. Note that $q_0^2 = o(1/L)$ and $p = \epsilon e^{-L}=o(1/L) $.
\begin{align}
    d_{KL}(q_0, p) \leq q_0L + q_0\log (q_0/\epsilon) - q_0 + q_0^2 + 2p \\ = \epsilon +  \frac{\epsilon}{L} (-\log L -1) + o(1/L) < \epsilon
\end{align}
for $p $ sufficiently small. Hence the point $q_0 $ is feasible, and $q^* \geq q_0 = \epsilon/\log(\epsilon/p) $.

Next, define $q_a = \frac{\epsilon}{L-a\log L}$ for some $a > 1 $. We will show that for sufficiently small $p $, we have $q^* < q_a $. 

Using the lower bound for $d_{KL}(q, p) $, we have:
\begin{align}
    d_{KL}(q_a, p) &\geq q_a(L + \log\frac{q_a}{\epsilon})-q_a-q_a^2 \\
    &= q_a(L-\log(L-a\log L))-q_a + o(1/L)
\end{align}

Noting that $\log(1-x) = -x + O(x^2) $ for small $x $, we have 
\begin{align}
    \log(L-a \log L) &=\log L + \log(1-a \log L /L) \\
    &=\log L -a\log L /L + O((\log L / L)^2).
\end{align}
Since $1/(1-x) = 1 + x +O(x^2) $
\begin{align}
    \frac{L-\log(L-a\log L)}{L-a\log L} &= \frac{L-\log L + a\log L/L + O((\log L /L)^2)}{L(1-a\log L /L))} \\
    &= (1-\frac{\log L}{L})(1 + a \frac{\log L}{L}) + O((1 /L)).
\end{align}
Therefore,
\begin{align}
    d_{KL}(q_a, p) &\geq \epsilon(1+ \frac{(a-1)\log L}{L} + O(1/L)) - \epsilon /L + O(1/L^2) \\
    &= \epsilon + \epsilon (\frac{(a-1)\log L-1}{L})+O(1/L) > \epsilon
\end{align}
when $L $ is large enough, i.e. $p $ is small enough. Hence $q^* < \frac{\epsilon}{L-a\log L} $. 

Since we have $\epsilon /L \leq q^*\leq \epsilon / (L-a\log L) $, the result follows, i.e. $q^*(p) \rightarrow \epsilon/L $ as $p \rightarrow 0 $.

\subsection{Proof of Lemma \ref{lem: bernoulli}}

Let $T_h $ be the time at which the $h^{th} $ success is achieved. We have $T_h $ = $\sum_{i=1}^h X_i  $, where $X_i \sim \text{Geom}(\frac{c}{i-1 +M_0}) $ is the time that passes from the $(i-1)^{th} $ success until $i^{th} $ success, i.e. $X_i  = T_i - T_{i-1} $, with $T_0  = 0 $. Note that $T_h $ is the sum of $h $ independent geometric random variables. We state the following concentration result from \citep{janson2018tail}:

\begin{lemma}[\citep{janson2018tail}] \label{lem: concent}
    Consider $X = \sum_{i=1}^n X_i $, where $X_i \sim Geom(p_i) $,  let $p^* = \min_i p_i $, and $\mu = E[X] $. For any $\lambda \geq 1 $,
    \begin{equation}
       \max\{\ \mathbb{P}(X \geq \lambda \mu), \mathbb{P}(X \leq \mu/\lambda) \} \leq e^{-p^*\mu(\lambda-1-\ln\lambda)}.
    \end{equation}
\end{lemma}

We have
\begin{align}
    E[T_h] &= \sum_{i=1}^h E[X_i] \\
    &= \sum_{i=1}^h \frac{i-1+M_0}{c} \\
    &= \frac{h(h-1)}{2c} + \frac{M_0h}{c}.
\end{align}

Let $\mu_h = E[T_h] $. Using Lemma \ref{lem: concent}, for any $\lambda \geq 1 $,
\begin{equation}
     \max\{\ \mathbb{P}(T_h \geq \lambda \mu_h), \mathbb{P}(T_h \leq \mu_h/\lambda) \} \leq e^{-\frac{h(h-1)}{2(M_0+h-1)}(\lambda-1-\ln \lambda)}.
\end{equation}

Note that for $t \in \mathbb{N} $, we have $\mathbb{P}(T_h \geq t) = \mathbb{P}(M_t \leq h) $ and $\mathbb{P}(T_h \leq t) = \mathbb{P}(M_t \geq h) $. Hence
\begin{equation}
    \mathbb{P}(T_h \geq \lceil \lambda\mu_h \rceil) = \mathbb{P}(M_{\lceil \lambda\mu_h \rceil} \leq h) \leq e^{-\frac{h(h-1)}{2(M_0+h-1)}(\lambda-1-\ln \lambda)}.
\end{equation}
Letting $t = \lceil \lambda \mu_h \rceil$, we have
\begin{equation}
    h \geq \frac{-M_0 + \sqrt{M_0^2 + 8c(t-1)/\lambda}}{2}
\end{equation}
and
\begin{equation}
    \mathbb{P}(M_t \leq \frac{-M_0 + \sqrt{M_0^2+ 8c(t-1)/\lambda}}{2}) \leq e^{-\Theta(\sqrt{t})}
\end{equation}
Letting $f_{\lambda}(t) = \frac{-M_0 + \sqrt{M_0^2 8c(t-1)/\lambda}}{2} $, we have 
\begin{equation}
    \sum_{t=1}^{\infty} \mathbb{P}(M_t \leq f_{\lambda}(t)) <\infty.
\end{equation}
Thus by Borel-Cantelli Lemma, $\mathbb{P}(M_t \leq f_{\lambda}(t) \text{ i.o.}) = 0 $. Since this is true for any $\lambda > 1 $, we have
\begin{equation}
    \liminf_{t\rightarrow \infty} \frac{M_t}{\sqrt{t}} \geq \sqrt{2c} \quad \text{a.s.}.
\end{equation}
Following the same steps for the upper bound yields
\begin{equation}
    \limsup_{t \rightarrow \infty} \frac{M_t}{\sqrt{t}} \leq \sqrt{2c} \quad \text{a.s.} .
\end{equation}
Hence we have shown that
\begin{equation}
    \lim_{t\rightarrow\infty} \frac{M_t}{\sqrt{2ct}} = 1 \quad \text{a.s.} .
\end{equation}

\subsection{Proof of Theorem \ref{thm: pull rate}}

To prove Theorem \ref{thm: pull rate}, we first introduce upper and lower bounds on the tails of the Gaussian distribution.

\begin{lemma}[\citep{wainwright2019high,cook2009upper}] \label{lem: gauss_tail}
    Let  $X \sim \mathcal{N}(0, \sigma^2) $ and $h > 0 $. Then we have
    \begin{equation}
       \frac{\sigma h}{h^2+\sigma^2} e^{-h^2/2\sigma^2} \leq \mathbb{P}(X > h) \leq \frac{\sigma}{h}e^{-h^2/2\sigma^2}.
    \end{equation}
\end{lemma}

We will reason about the growth of $N_{a,t} $ using (asymptotic) upper and lower bounds on $\pi_{a,t} $ of the form $g(N_{a,t}) \leq\pi_{a,t} \leq f(N_{a,t}) $. The following lemma can be used to show that the rate of growth of $N_{a,t} $ almost surely cannot be higher than the case in which $\pi_{a,t} = f(N_{a,t}) $, or lower than the case with $\pi_{a,t} = g(N_{a,t}) $.

\begin{lemma} \label{lemma:5}
    Consider a sequence of Bernoulli random variables $Z_n $, and let $\mathcal{H}_{t}$ denote the $\sigma $-field up to time $t $. Let $M_t = \sum_{n=1}^{t} Z_{t} $. Suppose that for some functions $f: \mathbb{N} \rightarrow (0,1) $ and $g: \mathbb{N} \rightarrow (0,1) $, we have $\mathbb{P}(Z_t = 1|\mathcal{H}_{t-1}) \leq f(M_{t-1}) $ and $\mathbb{P}(Z_t = 1 | \mathcal{H}_{t-1}) \geq g(M_{t-1}) $. 

    Consider two other sequences of Bernoulli random variables $Z_n^{up} $ and $Z_n^{down} $, which satisfy $\mathbb{P}(Z_t^{up}=1|\mathcal{H}_{t-1}) = f(M_{t-1}) $ and $\mathbb{P}(Z_t^{down}| \mathcal{H}_{t-1}) = g(M_{t-1}) $. %

 Then we have
 \begin{equation} 
     \mathbb{P}(M_t \geq h) \leq \mathbb{P}(M_t^{up} \geq h) \text{ and } \mathbb{P}(M_t \geq h) \geq \mathbb{P}(M_t^{down} \geq h).
 \end{equation}
    
\end{lemma}

\begin{proof}[Proof of Lemma \ref{lemma:5}]
    We will give the proof for the $\mathbb{P}(M_t \geq h) \leq \mathbb{P}(M_t^{up} \geq h)  $. The other statement follows from the same line of arguments. 

Consider the random variables $T_{h} = \min\{t:M_t = h \} $ and $T_{h}^{up} = \min \{t: M_t^{up} = h \} $. For each $h$, consider a sequence of independent uniform random variables $U_{h,1}, U_{h,2},... \sim \text{Uniform}[0,1] $. We can create a coupling of variables $(T_h, T_h^{up}) $ as follows: Define $W_h^{up} := \min \{k: U_{h,k} \leq f(h-1) \} $ and $W_h := \min \{k:U_{h,k} \leq p_{h,k}\} $, where $p_{h,k} = \mathbb{P}(Z_{T_{h-1} + k} = 1|\mathcal{H}_{T_h+k-1}) \leq f(h-1) $. Note that we have $W_h^{up} \leq W_h \text{ a.s.}$. Let $\hat T_h = \sum_{i=1 }^hW_i$ and $\hat T_h^{up} = \sum_{i=1}^h W_i^{up} $. Then we have $\hat T_h^{up} \leq \hat T_h \text{ a.s.} $. Since the $\hat T_h^{up} $ and $\hat T_h $ have the same marginal distributions as $T_h^{up} $ and $T_h $, respectively, it follows that
\begin{equation}
    \mathbb{P}(M_t \geq h )= \mathbb{P}(T_h \leq t) \leq \mathbb{P}(T_h^{up} \leq t) = \mathbb{P}(M_t^{up} \geq h).
\end{equation}
\end{proof}
An immediate corollary to Lemma 5 is the following:

\begin{lemma}\label{lem: growth}
    Suppose that for some functions $F(t) $ and $G(t) $ and any $\lambda > 0 $, we have
    \begin{equation}
        \sum_{t=1}^{\infty} \mathbb{P}(M_t^{up} > (1+ \lambda)F(t)) < \infty, \quad \sum_{t=1}^{\infty} \mathbb{P}(M_t^{down} < \frac{G(t)}{1 + \lambda}) < \infty.
    \end{equation}
    Then we have
    \begin{equation}
        \limsup_{T \rightarrow \infty} \frac{M_T}{F(T)} \leq 1, \quad \liminf_{T \rightarrow \infty} \frac{M_T}{G(T)} \geq 1 \quad \text{a.s.}.
    \end{equation}
\end{lemma}

\textbf{Proof of Theorem \ref{thm: pull rate}}

We denote by $N_{a, t_1:t_2} $ the number of pulls from arm $a $ between times $t_1 $ and $t_2 $: $N_{a, t_1:t_2} = \sum_{t=t_1}^{t_2} \mathbbm{1}_{A_t=a} $.

\textbf{Lower Bound on the Rate}: We will start by showing that $\liminf_{T \rightarrow \infty} \frac{N_{a,T}}{\sqrt{T}} \geq \sqrt{\frac{4\epsilon\sigma^2}{(\mu_1-\mu_a)^2}}$ a.s..

Recall that $\hat \mu_{i,t} $ denotes the empirical mean of arm $i $ based on observations before time $t $, and the posterior over the mean of arm $i $ at time $t $ is $\mathcal{N}(\hat{\mu}_{i,t}, \sigma^2/N_{i,t}) $. Hence, letting $X_{i,t} \sim \mathcal{N}(\hat{\mu}_{i,t}, \sigma^2/N_{i,t}) $, the probability assigned by $\pi_{\text{ref}} $ to arm $i $ at time $t $ is $\pi_{i,t} = \mathbb{P}(X_{i,t} = \max_j X_{j,t}) $. 

Let $\beta > 0 $, and let $\sigma_{1,t} = \sigma^2 / N_{1,t} $. Consider the event
$$B_t = \{X_{a,t} > \hat \mu_{1,t} + \beta \sigma_{1,t}, X_{j,t} < \hat \mu_{1,t} + \beta \sigma_1 \text{ for } j \neq a\} .$$
We have $\pi_{a,t} \geq \mathbb{P}(B_t) $. Due to independence, we have
\begin{equation}\label{eq: lowerbound}
     \pi_{a,t} \geq \mathbb{P}(X_{a,t} > \hat\mu_{1,t}+\beta \sigma_{1,t})\mathbb{P}(X_{1,t} < \hat \mu_{1,t} + \beta\sigma_{1,t}) \prod_{j \notin\{1, a \} } \mathbb{P}(X_{j,t} < \hat\mu_{1,t} + \beta\sigma_{1,t}). 
\end{equation}

Note that $\mathbb{P}(X_{1,t} < \hat{\mu}_{1,t} + \beta \sigma_{1,t} ) = \Phi(\beta)$ is a constant, where $\Phi(\cdot) $ is the cumulative density function of the standard normal distribution. Also note that, since each arm is pulled infinitely often, $\hat{\mu}_{i,t} \rightarrow \mu_i $ for all $i $ almost surely, hence the term
\begin{equation}
    c_t := \prod_{j \notin \{1,a\} } \mathbb{P}(X_{j,t} < \hat\mu_{1,t} + \beta \sigma_{1,t}) \rightarrow 1
\end{equation}
almost surely. 

Fix $\delta > 0 $ such that $\mu_a + \delta < \mu_1-\delta $, $c \in (0,1)$, and $n_a, n_1 \in \mathbb{N} $, where $n_a $ and $n_1 $ will be chosen later. Define the event
$$
E_t = \left\{ N_{a,t} > n_a, N_{1,t} > n_1, c_t > c, |\hat{\mu}_{j,t} -\mu_j| < \delta \text{ for } j \in \{1,a\}  \right\}
$$
and let $\tau := \inf\{t: E_{t'} \text{ holds } \forall t'>t \}$. $\tau $ is almost surely finite due to strong law of large numbers. 

We have $N_{a,T} \geq   N_{a, \tau: T} $. Since $E_t $ holds for $t \geq \tau $, using equation \eqref{eq: lowerbound}, we have 
\begin{equation}
    \pi_{a,t} \geq c\Phi(\beta) \mathbb{P}(X_{a,t} > \hat{\mu}_{1,t} + \beta \sigma_1)
\end{equation}
Using the lower bound in Lemma \ref{lem: gauss_tail} to lower bound $\mathbb{P}(X_{a,t} > \hat \mu_{1,t} + \beta \sigma_{1,t}) $, and noting that on $E_t $ we have $|\hat \mu_{1,t} - \hat \mu_{a,t}| \leq \mu_1-\mu_a + 2\delta $, yields
\begin{equation}
    \pi_{a,t} \geq \exp(-\frac{N_{a,t}( \mu_{1}-  \mu_{a} + 2 \delta +  \beta \sigma_{1,t})^2}{2\sigma^2} - \gamma(c, \beta, \delta) \log(N_{a,t}) ) \text{ for } t \geq \tau,
\end{equation}
where $\gamma(c, \beta, \delta) > 0 $ is a constant whose value is determined by $c, \beta, \delta $. For a given $\delta_a, \delta_1 > 0 $, we can choose $n_a $ such that $\frac{\log n_a}{n_a} < \delta_a $ and $\beta \sigma^2 /n_1 < \delta_1 $. It follows that, for $t \geq \tau $, we have
\begin{equation}
    \pi_{a,t} \geq \exp\left(-N_{a,t} \left( \frac{(\mu_1-\mu_a + \delta')^2}{2\sigma^2} + \delta_a \gamma(c, \beta, \delta) \right) \right)
\end{equation}
where $\delta' = \delta_1 + 2\delta $.

We can choose $n_a $ large enough so that we can use Lemma 1 to give a lower bound on the boosted probability. In this case, when the agent boosts arm $a $, we have
\begin{equation}
    \pi_{a,t}^b \geq \frac{\epsilon }{N_{a,t} \left( \frac{(\mu_1-\mu_a + \delta')^2}{2 \sigma^2} + \delta_a\gamma(c, \beta,\delta) \right)  + \log\epsilon}
\end{equation}

Using Lemma \ref{lem: bernoulli} and \ref{lem: growth}, it follows that, for any $\lambda > 1 $ the pulls from arm $ a$ after time $\tau $, denoted $N_{a,\tau:T} $ will fall below 
\begin{equation}
    \sqrt{\frac{\epsilon(T-\tau)/\lambda}{ \frac{(\mu_1-\mu_a + \delta')^2}{2 \sigma^2} + \delta_a\gamma(c, \beta,\delta)}}
\end{equation}
for finitely many values of $T $. As this is true for any $\delta' > 0, \delta_a > 0 $, and $\tau $ is almost surely finite for all $\delta' >0, \delta > 0 $, it follows that 
\begin{equation}
    \liminf_{T \rightarrow \infty} \frac{N_{a,T}}{\sqrt{T}} \geq \sqrt{\frac{4 \epsilon \sigma^2}{(\mu_1-\mu_a)^2}} \text{ a.s..}
\end{equation}

\textbf{Upper bound on the Rate:} We will show that $\limsup_{T \rightarrow \infty} \frac{N_{a,T}}{\sqrt{T}} \leq \sqrt{\frac{4 \epsilon \sigma^2}{(\mu_1-\mu_a)^2}} \text{ a.s.} $.

For an upper bound, we first need to show that the quantity $\frac{N_{1,T}}{N_{a,T}} $ diverges almost surely, given that both arm $1 $ and arm  $a $ are pulled infinitely often. 

\begin{lemma} \label{lemma:7}
    On the event that each arm is pulled infinite number of times, we have
    \begin{equation}
        \lim_{t \rightarrow\infty} \frac{N_{1,t}}{N_{a,t}} = \infty \quad \text{a.s.}.
    \end{equation}
\end{lemma}

\begin{proof}[Proof of Lemma \ref{lemma:7}]
   Since each arm is pulled infinite number of times, we will have $\pi_{1,t} \rightarrow 1 $, and for all $i \neq 1 $ $\pi_{i,t} \rightarrow 0 $ almost surely. Let $\pi_t^b $ denote the vector of pulling probabilities used by the agent at time $t $.  Note that $\pi_{i,t}^b \rightarrow 0$ and $\pi_{1,t}^b \rightarrow1 $ almost surely. Fix a $c > 1 $. Then for some almost surely finite time $\tau_c $, we will have $\pi_{1,t}^b > c\pi_{i,t}^b $ for $t \geq \tau_c $. It follows that $N_{1,t} - N_{i,t} \rightarrow \infty $ almost surely for all $ i \neq 1$. 

As $N_{1,t} $ has to grow at a linear rate, suppose that both $N_{1,t} $ and $N_{a,t} $ grow at a linear rate, i.e. there is $c_1, c_1' $ and $c_a, c_a' $ such that $c_1't \geq N_{1,t} \geq c_1t $ and $c_a't \geq  N_{a,t} \geq c_at $ for infinitely many $t $, i.e. for $t > t_0 $ for some $t_0 $. Then we have, for all $t > t_0 $, $N_{1,t} \geq N_{a,t}c_1/c_a' $, and consequently:
\begin{align}
    \pi_{a,t} &\leq \exp(- \frac{N_{1,t}N_{a,t}(\hat\mu_{1,t} - \hat \mu_{a,t})^2}{2\sigma^2(N_{1,t} + N_{a,t})}) \\
    & \leq \exp(- \frac{cN_{a,t}(\hat \mu_{1,t} - \hat \mu_{a,t})^2}{2 \sigma^2(c + 1)}),
\end{align}
where $c = c_1 / c_a' $. This in turn yields
\begin{equation}
    \pi_{a,t}^b \leq \frac{\epsilon}{\Theta(N_{a,t})},
\end{equation}
which implies, using Lemma \ref{lem: bernoulli} and Lemma \ref{lem: growth}, that the pulls from arm $a $ can grow at most at a $\Theta(\sqrt{T}) $ rate, which yields a contradiction.   
\end{proof}

Define the event $E_t = \{N_{1,t} / N_{a,t} > c, |\hat\mu_{i,t}-\mu_i| < \delta \text{ for } i \in \{1,a \} \} $, where $\delta > 0 $ is such that $\mu_a + \delta < \mu_1 - \delta $, and let $\tau = \min \{t: E_{t'} \text{ holds for all } t' \geq t \} $. We have $\tau < \infty $ almost surely, hence $\lim_{T \rightarrow \infty} \frac{N_{a, \tau}}{T} = 0$.

To upper bound $N_{a, \tau:T} $, we will first upper bound $\pi_{a,t} = \mathbb{P}(A_t=a| \mathcal{H}_{t-1}^{pub}) $ given that $E_t $ holds:
\begin{align}
    \pi_{a,t} &\leq \mathbb{P}(X_{a,t} > X_{1,t}) \\
    &\leq \exp(- \frac{N_{1,t}N_{a,t}(\hat \mu_{1,t} - \hat \mu_{a,t})^2}{2\sigma^2(N_{1,t} + N_{a,t})}) \\
    &\leq \exp(- \frac{cN_{a,t}(\mu_1-\mu_a-2\delta)^2}{2\sigma^2(c+1)}).
\end{align}
Using Lemma \ref{lem:kl_boost}, when $N_{a,t} $ is large enough the agent can increase the upper bound up to
\begin{equation}
    \pi_{a,t}^b \leq \frac{\gamma(N_{a,t}) \epsilon}{\frac{c(\mu_1-\mu_a-2\delta)^2N_{a,t}}{2\sigma^2(c+1)} + \log \epsilon}
\end{equation}
where $\gamma(N_{a,t}) \geq 1$ and  $\gamma(N_{a,t}) \rightarrow 1 $ as $N_{a,t} \rightarrow \infty $. It follows that, using lemma \ref{lem: bernoulli} and lemma \ref{lem: growth}, for $T > \tau $, for any $\lambda > 1 $, the pulls from arm $a $ after time $\tau $, denoted $N_{a, \tau:T} $, can exceed  
\begin{equation}
    \sqrt{\frac{4\epsilon \gamma \sigma^2(c + 1)T\lambda}{c(\mu_1-\mu_a-2\delta)^2}}.
\end{equation}
for only finitely many values of $T $. Since this is true for any $\gamma > 1, c > 1, \delta > 0 $,  it follows that
\begin{equation}
    \limsup_{T \rightarrow \infty} \frac{N_{a,T}}{\sqrt{T}} \leq \sqrt{\frac{4\epsilon\sigma^2}{(\mu_1-\mu_a)^2}} \quad \text{a.s.} .
\end{equation}

\section{Properties of the Problem \ref{eq: max_min_all}}

The analysis of Problem \ref{eq: max_min_all} follows the same lines of the proof of Theorem 5 in \citep{track}.

Recall that Problem \ref{eq: max_min_all} is
\begin{equation}
    \max_{w \in \mathcal{S}_K} \min\{\min_{a \neq 1} \frac{\sqrt{w_1w_a}\Delta_a^2}{\sqrt{w_1}d_a+\sqrt{w_a}d_1}, \frac{\sqrt{w_1}\Delta^2}{d_1}\}.
\end{equation}
To see that a solution exists, note that the minimum of continuous functions is continuous, and continuous functions attain their maxima on compact sets. 

Without loss of generality, we assume that $\mu_1^{priv} > \mu_2^{priv} \geq ....\geq \mu_K^{priv} $.

Let $w^* $ be an optimizer, and let $\gamma^* $ be the respective objective value. Define $x_a^* = w_a^*/w_1^*$. We have 
\begin{equation}
    w_1^* = \frac{1}{1 + \sum_{a=2}^Kx_a^*}, w_a^* = \frac{x_a^*}{1 +\sum_{a=2}^K x_a^*}
\end{equation}

Define the function $g_a(x) = \frac{\sqrt{x_a}\Delta_a^2}{d_a + \sqrt{x_a}d_1} $. $g_a(x) $ is strictly increasing as $g'_a(x) > 0 $. Noting that $g_a(0) = 0 $ and $g_a(x) \rightarrow \Delta_a^2/d_1 $ as $x \rightarrow \infty $, we can define the inverse function of $g_a $ as $x_a(y) =g_a^{-1}(y) $ on the interval $[0, \Delta_a^2/d_1) $.

The vector $(x_2^*, ...,x_K^*)$ is a solution to 
\begin{equation}
    \max_{(x_2,...,x_k)\in R^{K-1} } \frac{1}{\sqrt{1 + \sum_ax_a}} \min \{\min_{a\neq 1}g_a(x_a), \Delta^2/d_1 \}.
\end{equation}

First, we will show that $g_a(x_a^*) $ must all be equal, which is a restatement of the information balance property. 

Note that at the solution $x^* $, we have $\Delta^2/d_1 \geq g_a(x_a^*) $ for all $a $. Suppose otherwise, i.e. for some $a $, $g_a(x^*_a) > \Delta^2/d_1 $. Then it is possible to increase the objective value by slightly decreasing $x_a^* $, which yields a contradiction. 

Now we will show that $g_a(x_a^*) = \min_{i \neq1} g_i(x_i^*)$ for all $a =2,....K $. Consider the set $A=\{a: g_a(x_a^*)  > \min_{i\neq1} g_i(x_i^*)\} $. Suppose $A $ is not empty. Then it is possible to increase the objective value by slightly decreasing the $x_a $ values for $a \in A $, which results in a contradiction.  

Thus there exists a $y^* $ such that $g_a(x_a^*) = y^* $ for all $a=2,...K $ and $\Delta^2/d_1 \geq y^* $. Noting that $ x_a^* = g_a^{-1}(y^*) = x_a(y^*)$, we have 
\begin{equation}
    y^* \in \arg \max_{y \in [0, \Delta_2^2/d_1)} \frac{1}{\sqrt{1 + \sum_ax_a(y)}}\min\{y, \Delta^2/d_1 \} .
\end{equation}

Define $G(y) = \frac{y}{\sqrt{1+ \sum_ax_a(y)}}$. Note that the optimal solution $y^* $ satisfies $y^* \leq \Delta^2 /d_1 $. Equivalently, we can write
\begin{equation} \label{eq: maxG}
    y^* \in \arg \max_{y \in [0, \min\{\Delta^2,\Delta_2^2\}/d_1]} G(y).
\end{equation}

The function $G(y) $ is differentiable. $G'(y) = 0 $ yields
\begin{align}
    1+\sum_a x_a(y) = y\sum_ax_a'(y) /2.
\end{align}
Let $F(y) = 1 + \sum_a x_a(y) - y/2 \sum_{a}x_a'(y) $. Then we have $F(y^*)=0 $. Consider the derivative
\begin{equation}
    F'(y) = \frac{1}{2}\sum_a(x_a'(y)-yx_a''(y)) < 0.
\end{equation}
We have $F(0) = 1 $, and $F(y) \rightarrow -\infty  $ as $ y \rightarrow \Delta_2^2/d_1$. Consequently, the equation $F(y) = 0 $ has a unique solution, which is given by $y^* $. There are 2 possible cases:

\textbf{Case 1: $y^* \leq \Delta^2/d_1 $}. 

In that case, $y^* $ is feasible for the problem \eqref{eq: maxG}. Intuitively, in this case, the problem is equivalent to a $K-1 $ best arm identification problem, as the exploration performed for identifying the best among the $K-1 $ suboptimal public arms is also enough to decide whether $\mu_1 > \mu_{i^*} $.

Noting that $G(0) = 0 $ and $G(y) \rightarrow 0$ as $y \rightarrow \Delta_2^2/d_1 $,, the local optimum obtained by solving $F(y) = 0 $ is indeed a global optimum.

\textbf{Case 2. $y^* > \Delta^2/d_1 $}

In that case, $y^* $ is not feasible for problem \eqref{eq: maxG}. The optimal solution to \eqref{eq: maxG} is given by $ \Delta^2/d_1 $. Intuitively, in this case, the mean of arm $i^* $ is very close to that of arm 1 in terms of private rewards so that the exploration performed for the $K-1 $ best arm identification problem is not enough to declare that $\mu_1^{\text{priv}} > \mu_{i^*}^{priv} $.

\section{Implementation Details of Algorithm \ref{alg:ttts-bai}}

Algorithm 1 requires, at each time step, determining a leader arm and a challenger arm. This step is common among all top-two algorithms. In our implementations, we determined the top-two candidates by sampling from the private posterior $\Pi_t^{priv} $, but computationally simpler alternatives, such as the Top-Two Transportation Cost \citep{shang2020fixed}, are also possible. 

Different from the usual top-two algorithms, Algorithm \ref{alg:ttts-bai} requires two additional steps. The first is computing the reference probabilities assigned by $\pi_{\text{ref}} $, which in the case of Gaussian rewards, can be efficiently done with numerical integration as we discuss in Section \ref{sec: reference_probs}. The other additional step is solving the boosting problem in Equation \eqref{eq: boost}, which we discuss in Section \ref{sec: solve_boost}.

\subsection{Computing the Reference Probabilities} \label{sec: reference_probs}

Implementing Algorithm \ref{alg:ttts-bai} requires computing the pull probabilities assigned by $\pi_{\text{ref}} $ to each arm. This probability is given by 
\begin{equation} \label{eq: max_prob}
    \mathbb{P}(X_a = \max_i X_i) = \int_{- \infty }^{\infty} f_a(x) \prod_{j \neq a} \mathbb{P}(X_j \leq x) dx, 
\end{equation}
where $f_a(x) $ is the probability density function of $X_a $.

Since we are interested in Gaussian rewards, by making the change of variable $X_a  = \mu_a + \sigma_a Z$, where $Z \sim \mathcal{N}(0,1) $, the expression in equation \eqref{eq: max_prob} can be equivalently written as 
\begin{equation} \label{eq: max_prob2}
    \int_{- \infty}^{\infty} \frac{1}{\sqrt{2\pi}} e^{-x^2 /2}\prod_{j \neq a } \left[ \Phi\left(\frac{\mu_a-\mu_j+ \sigma_ax}{\sigma_j}   \right) \right ]dx,
\end{equation}
where $\Phi(\cdot ) $ is the cumulative density function of the standard normal distribution.

To approximate the integral in \eqref{eq: max_prob2}, we used the Gauss-Hermite quadrature \citep{abramowitz1965handbook} with $n $ nodes, which yields
\begin{equation}
    \mathbb{P}(X_a = \max_iX_i) \approx \frac{1}{\sqrt{\pi}} \sum_{q=1}^n w_q \prod_{j \neq a } \left[ \Phi\left(\frac{\mu_a-\mu_j+ \sigma_a\sqrt{2} x_q}{\sigma_j}   \right) \right ],
\end{equation}
where $x_q $ are the roots of the Hermite polynomial of degree $n $, and $w_q $ are the corresponding weights.

In our experiments, we used 32 nodes to compute the approximation.

\subsection{Solving the Boosting Problem}\label{sec: solve_boost}

\begin{figure}[h] \label{fig: boost}
    \centering
\includegraphics[width=0.45\linewidth]{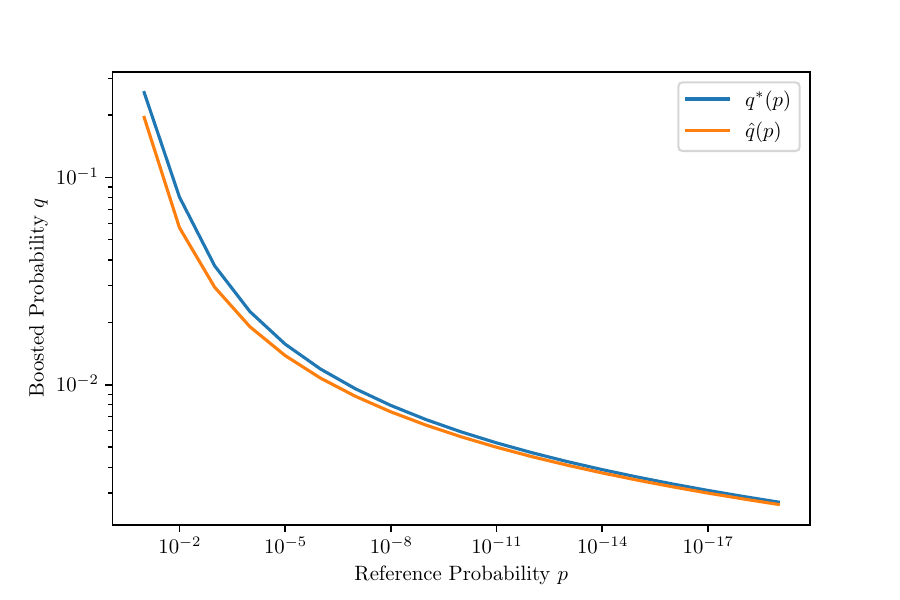}
    \caption{Comparison of $q^*(p) $ and $\hat q(p) $ for $\epsilon=0.1 $.}
    \label{fig:placeholder}
\end{figure}

Algorithm \ref{alg:ttts-bai} requires solving a KL-constrained maximization problem at each time step, which is given in \eqref{eq: boost}. As mentioned in the proof of Lemma \ref{lem: bernoulli}, this problem can be reduced to the following single-parameter problem:
\begin{align}
    q^*(p) = \max_{q \in (0,1]} \quad   &q \\
    s.t. \quad &d_{KL}(Ber(q), Ber(p)) \leq \epsilon,    
\end{align}
and can be solved efficiently with bisection. 

To further speed up our experiments, we instead approximated $q^*(p) $ using the following analytical approximation, which is a global underestimator of $q^*(p) $, and becomes exact as $p \rightarrow 0 $:
\begin{equation}
    \hat q(p) = \max \left\{\max \left\{ \frac{ \epsilon}{W(\epsilon/p)}, p \right\}, p + \sqrt{\epsilon p (1-p) }  \right\},
\end{equation}
where $W(\cdot) $ is the Lambert W function.

\end{document}